\documentclass[10pt,twocolumn,twoside]{IEEEtran}

\usepackage{amssymb}
\usepackage{multirow}
\usepackage{booktabs}
\usepackage{amsthm}

\usepackage{amsthm}
\newtheorem{theorem}{Theorem}[section]

\newtheorem{lemma}[theorem]{Lemma}
\newtheorem{definition}{Definition}[section]
\usepackage{empheq} 
\usepackage{mathtools}
\usepackage{xcolor}
\usepackage{mathrsfs}
\usepackage{amsfonts}
\usepackage{amsmath}
\usepackage{amssymb}
\usepackage{graphicx}
\usepackage{fancybox}
\usepackage{float}
\usepackage{cite}
\usepackage{enumitem}

\usepackage{gensymb}
\usepackage{multicol}
\usepackage{subfigure}
\usepackage{algorithm,algcompatible,amsmath}

\makeatother

\title{RAN-GNNs: breaking the capacity limits of graph neural networks}

\author{Diego Valsesia,~\IEEEmembership{Member,~IEEE},
        Giulia Fracastoro,~\IEEEmembership{Member,~IEEE},
        Enrico Magli,~\IEEEmembership{Fellow,~IEEE} 
\thanks{Diego Valsesia and Giulia Fracastoro contributed equally to this work. The authors are with Politecnico di Torino - Department of Electronics and Telecommunications,  Italy.  Email:{name.surname}@polito.it.}  }

\begin{document}
\maketitle

\begin{abstract}
Graph neural networks have become a staple in problems addressing learning and analysis of data defined over graphs. However, several results suggest an inherent difficulty in extracting better performance by increasing the number of layers. Recent works attribute this to a phenomenon peculiar to the extraction of node features in graph-based tasks, i.e., the need to consider multiple neighborhood sizes at the same time and adaptively tune them. In this paper, we investigate the recently proposed randomly wired architectures in the context of graph neural networks. Instead of building deeper networks by stacking many layers, we prove that employing a randomly-wired architecture can be a more effective way to increase the capacity of the network and obtain richer representations. We show that such architectures behave like an ensemble of paths, which are able to merge contributions from receptive fields of varied size. Moreover, these receptive fields can also be modulated to be wider or narrower through the trainable weights over the paths. We also provide extensive experimental evidence of the superior performance of randomly wired architectures over multiple tasks and four graph convolution definitions, using recent benchmarking frameworks that addresses the reliability of previous testing methodologies.
\end{abstract}

\section{Introduction} \label{sec:intro}

\begin{figure*}[t]
    \centering
    \includegraphics[width=0.9\textwidth]{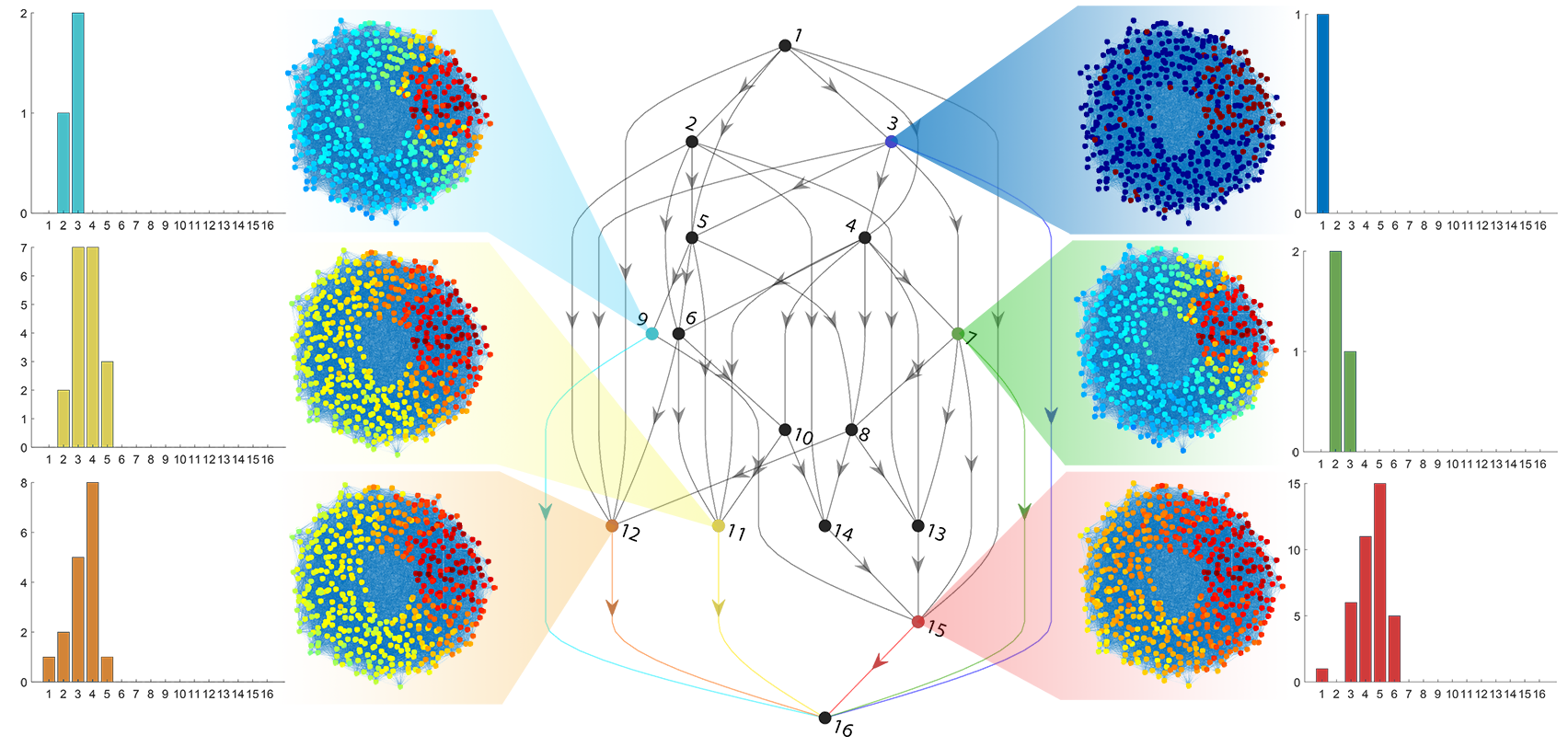}
    \caption{Random architectures aggregate ensembles of paths. This creates a variety of receptive fields (effective neighborhood sizes on the domain graph) that are combined to compute the output. Figure shows the domain graph where nodes are colored (red means high weight, blue low weight) according to the receptive field weighted by the path distribution of a domain node. The receptive field is shown at all the architecture nodes directly contributing to the output. Histograms represent the distribution of path lengths from source to architecture node.}
    \label{fig:full}
\end{figure*}

Data defined over the nodes of graphs are ubiquitous. Social network profiles \cite{hamilton2017inductive}, molecular interactions \cite{duvenaud2015convolutional}, citation networks \cite{sen2008collective}, 3D point clouds \cite{simonovsky2017dynamic} are just examples of a wide variety of data types where describing the domain as a graph allows to encode constraints and patterns among the data points. Exploiting the graph structure is crucial in order to extract powerful representations of the data. However, this is not a trivial task and only recently graph neural networks (GNNs) have started showing promising approaches to the problem. GNNs  \cite{wu2020comprehensive} extend the deep learning toolbox to deal with the irregularity of the graph domain. Much of the work has been focused on defining a graph convolution operation \cite{bronstein2017geometric}, i.e., a layer that is well-defined over the graph domain but also retains some of the key properties of convolution such as weight reuse and locality.
A wide variety of such graph convolution operators has been defined over the years, mostly based on neighborhood aggregation schemes where the features of a node are transformed by processing the features of its neighbors. Such schemes have been shown to be as powerful as the Weisfeiler-Lehman graph isomorphism test \cite{weisfeiler,xu2018powerful}, enabling them to simultaneuosly learn data features and graph topology.

However, contrary to classic literature on CNNs, few works \cite{li2019deepgcnsjournal,dehmamy2019understanding,xu2018representation,dwivedi2020benchmarking} addressed GNNs architectures and their role in extracting powerful representations. Several works, starting with the early GCN \cite{kipf2016semi}, noticed an inability to build deep GNNs, often resulting in worse performance than that of methods that disregard the graph domain, when trying to build anything but very shallow networks. This calls for exploring whether advances on CNN architectures can be translated to the GNN space, while understanding the potentially different needs of graph representation learning.

Li et al. \cite{li2019deepgcns} suggest that GCNs suffer from oversmoothing as several layers are stacked, resulting in the extraction of mostly low-frequency features. This is related to the lack of self-loop information in this specific graph convolution. It is suggested that ResNet-like architectures mitigate the problem as the skip connections supply high frequency contributions. Xu et al. \cite{xu2018representation} point out that the size of the receptive field of a node, i.e., which nodes contribute to the features of the node under consideration, plays a crucial role, but it can vary widely depending on the graph and too large receptive fields may actually harm performance. They conclude that for graph-based problems it would be optimal to learn how to adaptively merge contributions from receptive fields of multiple size. For this reason they propose an architecture where each layer has a skip connection to the output so that contributions at multiple depths (hence sizes of receptive fields) can be merged.  Nonetheless, the problem of finding methods for effectively increasing the capacity of graph neural networks is still standing, since stacking many layers has been proven to provide limited improvements \cite{li2019deepgcns, oono2019graph, alon2020bottleneck, nt2019revisiting}.

In this paper, we argue that the recently proposed randomly wired architectures \cite{xie2019exploring} are ideal for GNNs. In a randomly wired architecture, ``layers'' are arranged according to a random directed acyclic graph and data are propagated through the paths towards the output. Such architecture is ideal for GNNs because it realizes the intuition of \cite{xu2018representation} of being able of merging receptive fields of varied size. Indeed, the randomly wired GNNs (RAN-GNNs) can be seen as an extreme generalization of their jumping network approach where layer outputs can not only jump to the network output but to other layers as well, continuously merging receptive fields. Hence, randomly wired architectures provide a way of effectively scaling up GNNs, mitigating the depth problem and creating richer representations. Fig. \ref{fig:full} shows a graphical representation of this concept by highlighting the six layers directly contributing to the output, having different receptive fields induced by the distribution of paths from the input.

Our novel contributions can be summarized as follows: i) we are the first to analyze randomly wired architectures and show that they are generalizations of ResNets when looked at as ensembles of paths \cite{veit2016residual}; ii) we show that path ensembling allows to merge receptive fields of varied size and that it can do so \textit{adaptively}, i.e., trainable weights on the architecture edges can tune the desired size of the receptive fields to be merged to achieve an optimal configuration for the problem; iii) we introduce improvements to the basic design of randomly wired architectures by optionally embedding a path that sequentially goes through all layers in order to promote larger receptive fields when needed, and by presenting MonteCarlo DropPath, which decorrelates path contributions by randomly dropping architecture edges; iv) we provide extensive experimental evidence, using recently introduced benchmarking frameworks \cite{dwivedi2020benchmarking, hu2020ogb} to ensure significance and reproducibility, that randomly wired architectures consistently outperform ResNets, often by large margins, for four of the most popular graph convolution definitions on multiple tasks.

\section{Background}

\subsection{Graph Neural Networks}

A major shortcoming of CNNs is that they are unable to process data defined on irregular domains. In particular, one case that is drawing attention is when the data structure can be described by a graph and the data are defined as vectors on the graph nodes. This setting can be found in many applications, including 3D point clouds \cite{wang2019dynamic,valsesia2018learning}, computational biology \cite{alipanahi2015predicting, duvenaud2015convolutional}, and social networks \cite{kipf2016semi}. However, extending CNNs from data with a regular structure, such as images and video, to graph-structured data is not straightforward if one wants to preserve useful properties such as locality and weight reuse.

GNNs redefine the convolution operation so that the new layer definition can be used on domains described by graphs. The most widely adopted graph convolutions in the literature rely on message passing, where a weighted aggregation of the feature vectors in a neighborhood is computed. The GCN \cite{kipf2016semi} is arguably the simplest definition, applying the same linear transformation to all the node features, followed by neighborhood aggregation and non-linear activation:
\begin{align*}
    \mathbf{h}^{(l+1)}_i = \sigma\left( \frac{1}{\vert \mathcal{N}_i \vert} \sum_{j \in \mathcal{N}_i} \mathbf{W}\mathbf{h}^{(l)}_j \right).
\end{align*}
Variants of this definition have been developed, e.g., GraphSage \cite{hamilton2017inductive} concatenates the feature vector of node $i$ to the feature vectors of its neighbors, so that self-information can also be exploited; GIN \cite{xu2018powerful} uses a multilayer perceptron instead of a linear transform, replaces average with sum to ensure injectivity and proposes a different way of computing the output by using all the feature vectors produced by the intermediate layers. These definitions are all isotropic because they treat every edge in the same way. It has been observed that better representation capacity can be achieved using anistropic definitions, where every edge can have a different transformation, at the cost of increased computational complexity. The Gated GCN \cite{bresson2017residual} and GAT \cite{velivckovic2017graph} definitions fall in this category.

\subsection{Randomly wired architectures}

In recent work, Xie et al. \cite{xie2019exploring} explore whether it is possible to avoid handcrafted design of neural network architectures and, at the same time, avoid expensive neural architecture search methods \cite{elsken2019neural}, by designing random architecture generators. They show that ``layers'' performing convolution, normalization and non-linear activation can be connected in a random architecture graph. Strong performance is observed on the traditional image classification task by outperforming state-of-the-art architectures. The authors conjecture that random architectures generalize ResNets and similar constructions, but the underlying principles of their excellent performance are unclear, as well as whether the performance translates to tasks other than image recognition or to operations other than convolution on grids.

\section{Randomly wired GNNs}

In this section, we first introduce randomly wired graph neural networks (RAN-GNNs) and the notation we are going to use. We then analyze their behavior when viewed as ensembles of paths.

\begin{figure}
  \centering
  \includegraphics[width=0.25\textwidth]{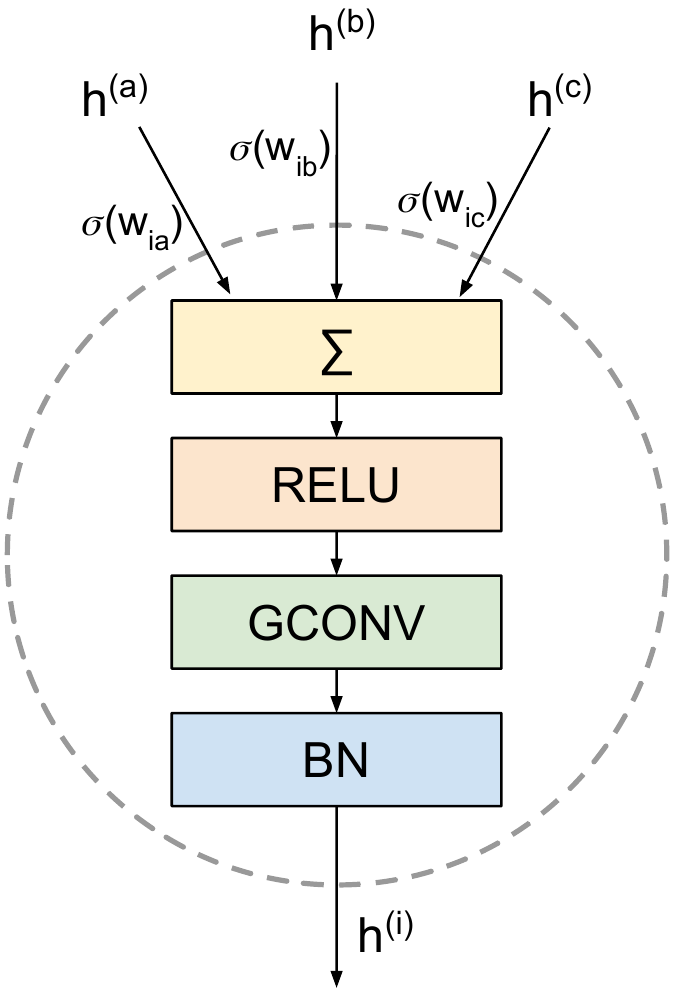}
    \caption{An architecture node is equivalent to a GNN layer.}
  \label{fig:architecture_node}
\end{figure}

A randomly wired architecture consists of a directed acyclic graph (DAG) connecting a source architecture node, which is fed with the input data, to a sink architecture node. One should not confuse the architecture DAG with the graph representing the GNN domain: to avoid any source of confusion we will use the terms \textit{architecture nodes} (edges) and \textit{domain nodes} (edges), respectively. A domain node is a node of the graph that is fed as input to the GNN. An architecture node is effectively a GNN layer performing the following operations (Fig. \ref{fig:architecture_node}): i) aggregation of the inputs from other architecture nodes via a weighted sum as in \cite{xie2019exploring}:
\begin{align}
    \mathbf{h}^{(i)} = \sum_{j \in \mathcal{A}_i} \omega_{ij} \mathbf{h}^{(j)} = \sum_{j \in \mathcal{A}_i} \sigma(w_{ij}) \mathbf{h}^{(j)} , \quad i=1,...,L-1
    \label{eq:aggr}
\end{align}
being $\sigma$ a sigmoid function, $\mathcal{A}_i$ the set of direct predecessors of the architecture node $i$, and $w_{ij}$ a scalar trainable weight; ii) a non-linear activation; iii) a graph-convolution operation (without output activation); iv) batch normalization.

The architecture DAG is generated using a random graph generator. In this paper, we will focus on the Erd\H{o}s-Renyi model where the adjacency matrix of the DAG is a strictly upper triangular matrix with entries being realizations of a Bernoulli random variable with probability $p$. If multiple input architecture nodes are randomly generated, they are all wired to a single global input. Multiple output architecture nodes are averaged to obtain a global output. Other random generators may be used, e.g., small-world and scale-free random networks have been studied in \cite{xie2019exploring}. However, a different generator will display a different behavior concerning the properties we study in Sec. \ref{sec:gradient_analysis}. 

\subsection{Randomly wired architectures behave like path ensembles}
\label{sec:gradient_analysis}

It has already been shown that ResNets behave like ensembles of relatively shallow networks, where one can see the ResNet architecture as a collection of paths of varied lengths \cite{veit2016residual}. More specifically, in a ResNet with $n$ layers, where all layers have a skip connection except the first one and the last one, there are exactly $2^{L-2}$ paths, whose lengths follow a Binomial distribution (i.e., the number of paths of length $l$ from layer $k$ to the last layer is $\binom{L-k-1}{l-2}$), and the average path length is $\frac{L}{2}+1$ \cite{veit2016residual}. In this section, we show that a randomly wired neural network can also be considered as an ensemble of networks with varied depth. However, in this case, the distribution of the path length is different from the one obtained with the ResNet, as shown in the following lemma.
\begin{lemma}
Let us consider a randomly wired network with $L$ architecture nodes, where the architecture DAG is generated according to an Erd\H{o}s-Renyi graph generator with probability $p$. The average number of paths of length $l$ from node $k$ to the sink, where $k<L$, is $\mathbb{E}[N_l^{(k)}]=\binom{L-k-1}{l-2}p^{l-1}$ and the average total number of paths from node $k$ to the sink is $\mathbb{E}[N^{(k)}]=p(1+p)^{L-k-1}$. 
\label{lemma1}
\end{lemma}
\begin{proof}
Let us first consider the number of paths of length $l$ from node $k$ to the sink. We define the path length as the number of nodes in the path.  In a randomly wired network with $n$ architecture nodes, we have that the first node of all the paths is node $k$ and the last one is node $n$ (i.e., the sink node). Therefore, the minimum path length is 2. If $l\ge 2$, the number of all possible paths of length $l$ between node $k$ and the sink is $\binom{n-k-1}{l-2}$. Since in a path of length $l$ there are $l-1$ edges and each edge has probability $p$ of being generated by the Erd\H{o}s-Renyi model, each one of the paths of length $l$ has probability $p^{l-1}$ of being present in the network. Thus, the expected number of paths with length $l$ between node $k$ and the sink is $\mathbb{E}[N_l^{(k)}]=\binom{n-k-1}{l-2}p^{l-1}$. If we set $k=1$, we obtain the average number of paths of length $l$ from source to sink $\mathbb{E}[N_l]=\binom{n-2}{l-2}p^{l-1}$. We can now compute the average total number of paths $\mathbb{E}[N^{(k)}]$ as follows
\[
\begin{split}
\mathbb{E}[N^{(k)}]&=\sum_{l=2}^{n-k+1} \binom{n-k-1}{l-2}p^{l-1} =\sum_{\tilde l=0}^{\tilde n} \binom{\tilde n}{\tilde l}p^{\tilde l+1}\\
&= p\sum_{\tilde l=0}^{\tilde n} \binom{\tilde n}{\tilde l}p^{\tilde l} =p(1+p)^{\tilde n}=p(1+p)^{n-k-1},
\end{split}
\]
where $\tilde n=n-k-1$, $\tilde l=l-2$ and the fourth equality follows from the binomial theorem. If we set $k=1$, we obtain the average total number of paths from source to sink $\mathbb{E}[N_p]=p(1+p)^{n-2}$.
\end{proof}
We can observe that if $p=1$, the randomly wired network converges to the ResNet architecture. This allows to think of randomly wired architectures as generalizations of ResNets as they enable increased flexibility in the number and distribution of paths, instead of enforcing the use of all $2^{L-2}$ paths.

\subsection{Receptive field analysis} \label{sec:receptive}

In the case of GNNs, we define the receptive field of a domain node as the neighborhood that affects the output features of that node. As discussed in Sec. \ref{sec:intro}, the work in \cite{xu2018representation} highlights that one of the possible causes of the depth problem in GNNs is that the size of the receptive field is not adaptive and may rapidly become excessively large. Inspired by this observation, in this section we analyze the receptive field of a randomly wired graph neural network. We show that the receptive field of the output is a combination of the receptive fields of shallower networks, induced by each of the paths. This allows to effectively merge the contributions from receptive fields of varied size. Moreover, we show that the trainable parameters along the path edges modulate the contributions of various path lengths and enable adaptive receptive fields, that can be tuned by the training procedure.

We first introduce a definition of the receptive field of a feedforward graph neural network\footnote{We use the term ``feedforward neural network'' to indicate an architecture made of a simple line graph, without skip connections: this is a representation of one path.}.
\begin{definition}
Given a feedforward graph neural network with $L$ layers, the receptive field of radius $L$ of a domain node is its $L$-hop neighborhood. 
\end{definition} 
In a randomly wired architecture, each path induces a corresponding receptive field whose radius depends on the length of the path. Then, the receptive field at the output of the network is obtained by combining the receptive fields of all the paths. In order to analyze the contribution of paths of different lengths to the receptive field of the network, we introduce the concept of distribution of the receptive field radius of the paths. 
Notice that if we consider a feedforward network with $L$ layers, the distribution of the receptive field radius is a delta centered in $L$.

The following lemma allows to analyze the distribution of the receptive field radius in a randomly wired architecture.
\begin{lemma}
\label{lemma:derivative}
The derivative $\frac{\partial y}{\partial x_0}$ of the output $y$ of a randomly wired architecture with respect to the input $x_0$ is
\begin{equation}\label{eq:path_sum}
\frac{\partial y}{\partial x_0}=\sum_{p\in\mathcal{P}}\frac{\partial y_p}{\partial x_0}=\sum_{p\in\mathcal{P}}\prod_{\{i,j\}\in\mathcal{E}^{p}}\omega_{ij}\frac{\partial \bar{y}_p}{\partial x_0}= \sum_{l=2}^L\sum_{p\in\mathcal{P}^l}\lambda_p\frac{\partial \bar{y}_p}{\partial x_0},
\end{equation}
where $y_p$ is the output of path $p$, $\bar{y}_p$ is the output of path $p$ when we consider all the aggregation weights equal to 1, $\lambda_p=\frac{\partial y_p}{\partial x_0}/\frac{\partial \bar{y}_p}{\partial x_0}$,  $\mathcal{P}$ is the set of all paths from source to sink, $L$ is the number of architecture nodes, $\mathcal{P}^l$ is the set of paths from source to sink of length $l$ and $\mathcal{E}^p$ is the set of edges of the path $p$.
\end{lemma}
\begin{proof}
Direct computation.
\end{proof}

From Lemma \ref{lemma:derivative}, we can observe that the contribution of each path to the gradient is weighted by its corresponding architecture edge weights.
Thus, we can define the following distribution  $\rho$ of the receptive field radius:
\begin{equation}
\rho_l=\sum_{p\in\mathcal{P}^l}\lambda_p=\sum_{p\in\mathcal{P}^l}\prod_{\{i,j\}\in\mathcal{E}^{p}}\omega_{ij}\qquad \mathrm{for }\ \ l=2,...,n,
\label{eq:rho}
\end{equation}
where we have assumed that the gradient $\frac{\partial \bar{y}_p}{\partial x_0}$ depends only on the path length, as done in \cite{veit2016residual}. This is a reasonable assumption if all the architecture nodes perform the same operation. The distribution of the receptive field radius is therefore influenced by the architecture edge weights. Figure \ref{fig:path1} shows an example of how such weights can modify the radius distribution. If we consider $\omega_{ij}=1$ for all $i$ and $j$, we obtain that the radius distribution is equal to the path length distribution. In order to provide some insight into the role of parameter $p$ in the distribution of the receptive field radius, we focus on this special case and analyze the distribution of the path lengths in a randomly wired architecture by introducing the following Lemma.

\begin{lemma}
Let us consider a randomly wired network with $L$ architecture nodes, where the architecture DAG is generated according to a Erd\H{o}s-Renyi graph generator with probability $p$. The average length of the paths from node $k$ to the sink is $\mathbb{E}[l^{(k)}]\approx\frac{p}{1+p}(L-k-1)+2$.
\label{lemma:avg_length}
\end{lemma}
\begin{proof}
From Lemma 3.1, we can compute the average length of the paths from node $k$ to the sink as follows
\begin{equation}
\begin{split}
\mathbb{E}[l^{(k)}]&=\sum_{l=2}^{n-k+1}l\mathbb{E}\left[\frac{N_l^{k}}{N^{(k)}}\right]\approx\frac{\sum_{l=2}^{n-k+1}l\mathbb{E}[N_l^{(k)}]}{\mathbb{E}[N^{(k)}]}\\
&=\frac{\sum_{l=2}^{n-k+1}\binom{n-k-1}{l-2}p^{l-1}l}{p(1+p)^{n-k-1}},
\end{split}
\label{eq:l_m}
\end{equation}
where we have neglected the higher order terms \cite{elandt1980survival}. The numerator in \eqref{eq:l_m} can be computed as follows
\[
\begin{split}
\sum_{l=2}^{n-k+1}\hspace{-5pt}\binom{n-k-1}{l-2}p^{l-1}l &= \sum_{\tilde l=0}^{\tilde n} \binom{\tilde n}{\tilde l}p^{\tilde l+1}(\tilde l+2)\\
&=\sum_{\tilde l=0}^{\tilde n} \binom{\tilde n}{\tilde l}p^{\tilde l+1}\tilde l+2\sum_{\tilde l=0}^{\tilde n} \binom{\tilde n}{\tilde l}p^{\tilde l+1}\\
&=p^2\sum_{\tilde l=0}^{\tilde n} \binom{\tilde n}{\tilde l}p^{\tilde l-1}\tilde l + 2p\sum_{\tilde l=0}^{\tilde n} \binom{\tilde n}{\tilde l}p^{\tilde l}\\
&=p^2\tilde n(1+p)^{\tilde n -1}+2p(1+p)^{\tilde n}\\
&=p^2(n-k-1)(1+p)^{n-k-2}\\
&+2p(1+p)^{n-k-1},
\end{split}
\]
where $\tilde n=n-k-1$, $\tilde l=l-2$ and the fourth equality is obtained differentiating the binomial theorem with respect to $p$ .
Then, we obtain 
\[
\mathbb{E}[l^{(k)}]=\frac{p}{1+p}(n-k-1)+2.
\]
If we consider $k=1$, i.e. the sink, we obtain $\mathbb{E}[l]=\frac{p}{1+p}(n-2)+2$.
\end{proof}

Therefore, if $p=1$ and $\omega_{ij}=1$ for all $i$ and $j$ the radius distribution is a Binomial distribution centered in $\frac{L}{2}+1$ (as in ResNets), instead when $p<1$ the mean of the distribution is lower.  The path length distribution for different $p$ values is shown in Fig. \ref{fig:path2}. This shows that, differently from feedforward networks, the receptive field of ResNets and randomly wired architectures is a combination of receptive fields of varied sizes, where most of the contribution is given by shallow paths, i.e. smaller receptive fields. The parameter $p$ of the randomly wired neural network influences the distribution of the receptive field radius: a lower $p$ value skews the distribution towards shallower paths, instead a higher $p$ value skews the distribution towards longer paths. 

\begin{figure}[t]
    \centering
    \includegraphics[width=0.95\columnwidth]{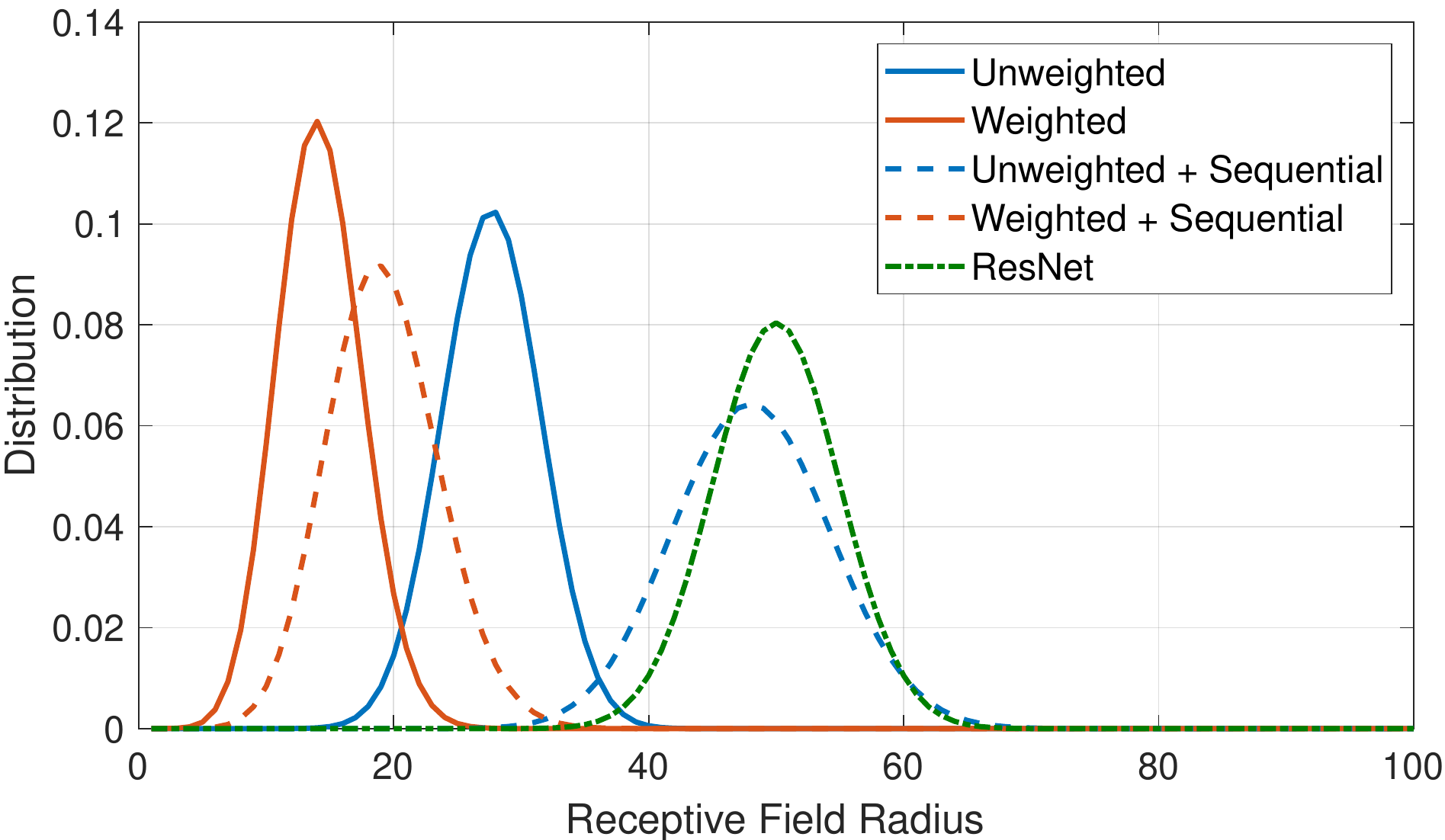}
    \caption{Distribution of receptive field radius ($p=0.4$, $\omega_{ij}=1$ for unweighted, $\omega_{ij}=0.5$ for weighted).}
    \label{fig:path1}
\end{figure}

\begin{figure}[t]
    \centering
    \includegraphics[width=0.95\columnwidth]{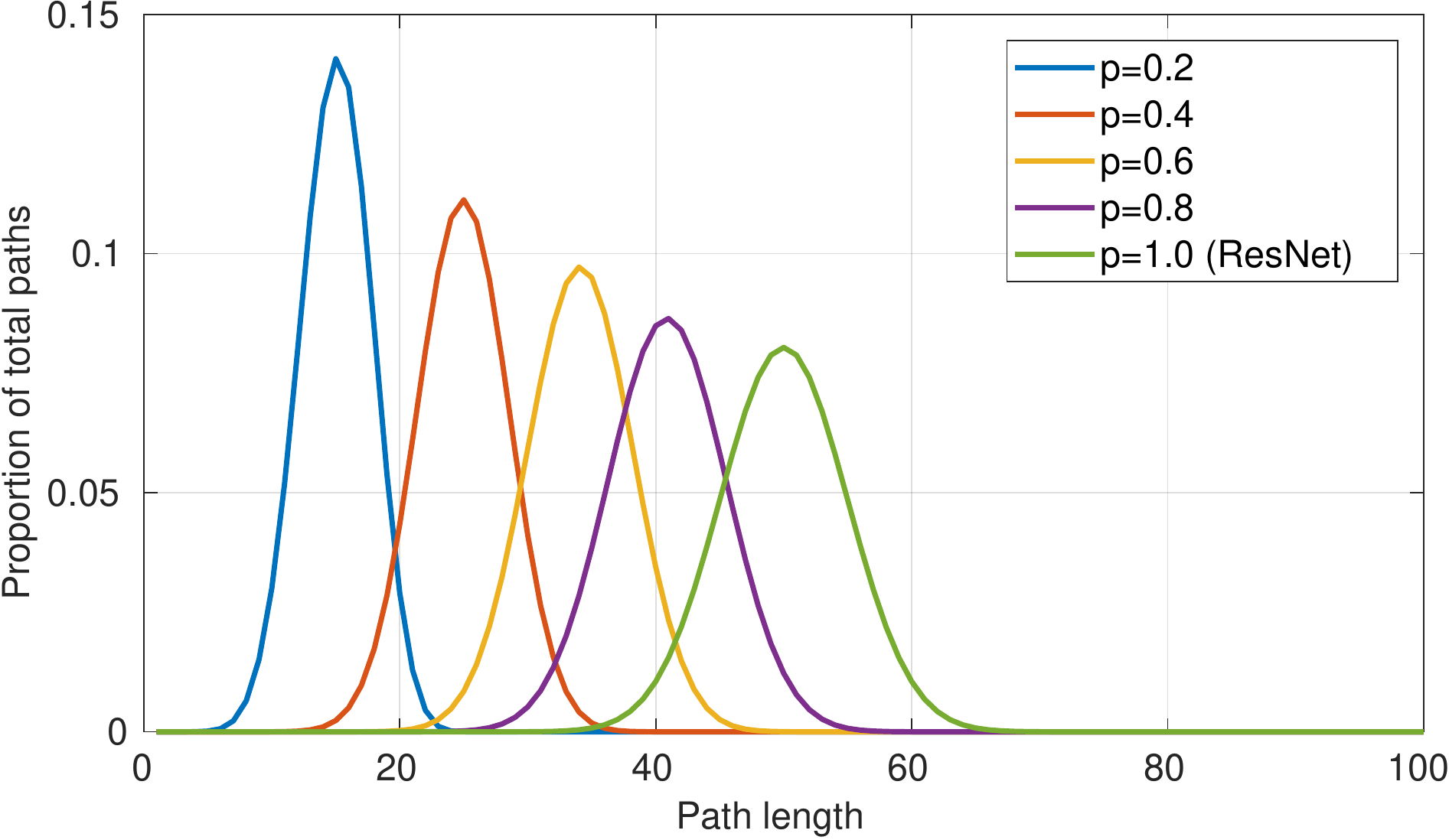}
    \caption{Path distribution as function of architecture edge probability.}
    \label{fig:path2}
\end{figure}

After having considered the special case where $\omega_{ij}=1$ for all $i$ and $j$, we now focus on the general case. Since the edge architecture weights are trainable parameters, they can be adapted to optimize the distribution of the receptive field radius. This is one of the strongest advantages provided by randomly wired architectures with respect to ResNets. This is particularly relevant in the context of GNNs, where we may have a non-uniform growth of the receptive field caused by the irregularity of the graph structure \cite{xu2018representation}. Notice that the randomly wired architecture can be seen as a generalization of the jumping knowledge networks proposed in \cite{xu2018representation}, where all the architecture nodes, not only the last one, merge contributions from previous nodes.
We also remark that, even if we modify the ResNet architecture by adding trainable weights to each branch of the residual module, we cannot retrieve the behaviour of the randomly wired architecture. In fact, the latter has intrinsically more granularity than a ResNet: the expected number of architecture edge weights of a randomly wired network is $\frac{pL(L+1)}{2}$, instead a weighted ResNet has only $2(L-2)$ weights. Ideally, we would like to weigh each path independently (i.e., directly optimizing the value of $\lambda_p$ in Eq. \eqref{eq:path_sum}). However, this is unfeasible because the number of parameters would become excessively high and the randomly wired architecture provides an effective tradeoff. Given an architecture node, weighting in a different way each input edge is important because to each edge corresponds a different length distribution of the paths going through such edge, as shown by the following Lemma. 

\begin{lemma}
Let us consider a randomly wired network with $n$ architecture nodes, where the architecture DAG is generated according to a Erd\H{o}s-Renyi graph generator with probability $p$. Given an edge $\{i,j\}$ between the architecture nodes $i$ and $j$ where $i<j$, the average length of the paths from the source to the sink going through that edge is $\mathbb{E}[l_{ij}]\approx\frac{p}{1+p}(L-(j-i)-3)+4$.
\end{lemma}
\begin{proof}
From Lemma 3.3, we can compute the average length of the paths going through the edge $\{i,j\}$ as follows
\[
\mathbb{E}[l_{ij}]=\mathbb{E}[l^{n-i+1}+l^{j}]\approx\frac{p}{1+p}(n-(j-i)-3)+4.
\]
\end{proof}

\subsection{Sequential path}
In the previous sections we have shown that a randomly wired architecture behaves like an ensemble of paths merging contribution from receptive fields of varied size, where most of the contribution is provided by shallow paths. As discussed previously, this provides numerous advantages with respect to feedforward networks and ResNets. However, some graph-based tasks may actually benefit from a larger receptive field \cite{li2019deepgcns}, so it is interesting to provide randomly wired architectures with mechanisms to directly promote longer paths. Differently from ResNets, in a randomly wired neural network with $L$ architecture nodes the longest path may be shorter than $L$, leading to a smaller receptive field. In order to overcome this issue, we propose to modify the generation process of the random architecture by imposing that it should also include the sequential path, i.e., the path traversing all architecture nodes. This design of the architecture skews the initial path length distribution towards longer paths, which has the effect of promoting their usage. Nevertheless, the trainable architecture edge weights will ultimately define the importance of such contribution. Fig. \ref{fig:path1} shows an example of how including the sequential path changes the distribution of the receptive field radius.

\subsection{MonteCarlo DropPath regularization} \label{sec:droppath}

The randomly wired architecture offers new degrees of freedom to introduce regularization techniques. In particular, one could delete a few architecture edges during training with probability $p_\text{drop}$ as a way to avoid co-adaptation of architecture nodes. This is reminiscent of DropOut \cite{srivastava2014dropout} and DropConnect \cite{wan2013regularization}, although it is carried out at a higher level of abstraction, i.e., connections between ``layers'' instead of neurons. It is also reminiscent of techniques used in Neural Architecture Search \cite{zoph2018learning} and the approach used in ImageNet experiments in \cite{xie2019exploring}, although implementation details are unclear for the latter. 

We propose to use a MonteCarlo approach where paths are also dropped in testing. Inference is performed multiple times for different realizations of dropped architecture edges and results are averaged. This allows to sample from the full predictive distribution induced by DropPath, as in MonteCarlo DropOut \cite{gal2015dropout}.
The following lemma shows that MonteCarlo DropPath decorrelates the contributions of paths in Eq. \eqref{eq:path_sum} even if they share architecture edges, thus allowing finer control over the modulation of the receptive field radius.
\begin{lemma}
Let us consider two distinct paths $p_1$ and $p_2$ of a randomly wired network where the edges of the paths can be deleted with probability $p_{\text{drop}}$. Then, even if the two paths share some architecture edges, their contributions to the derivative $\frac{\partial y}{\partial x_0}$, as
defined in Lemma \ref{lemma:derivative}, are decorrelated. \end{lemma}
\begin{proof}
Let us consider two paths $p_1$ and $p_2$ with at least one common edge, we can compute the covariance between these two paths as follows:
\begin{align*}
&\mathrm{Cov}(\lambda_{p_1},\lambda_{p_2})=\mathbb{E}[\lambda_{p_1}\lambda_{p_2}]-\mathbb{E}[\lambda_{p_1}]\mathbb{E}[\lambda_{p_2}]\\
&=\prod_{\{i,j\}\in \mathcal{E}^{p_1}}\hspace{-5pt} \omega_{ij}\prod_{\{i,j\}\in \mathcal{E}^{p_2}}\hspace{-5pt}\omega_{ij} \mathbb{E}\Bigg[\prod_{\{i,j\}\in \mathcal{I}(p_1,p_2)}\hspace{-3pt}z_{ij}^2 \prod_{\{i,j\}\in \mathcal{D}(p_1,p_2)}\hspace{-5pt}z_{ij}\Bigg]\\
&-\prod_{\{i,j\}\in \mathcal{E}^{p_1}}\hspace{-5pt} \omega_{ij}\prod_{\{i,j\}\in \mathcal{E}^{p_2}} \hspace{-5pt}\omega_{ij} \mathbb{E}\hspace{-3pt}\left[\prod_{\{i,j\}\in \mathcal{E}^{p_1}}\hspace{-5pt}z_{ij}\right]\mathbb{E}\hspace{-3pt}\left[\prod_{\{i,j\}\in \mathcal{E}^{p_2}}\hspace{-5pt}z_{ij}\right]\hspace{-5pt}=0,
\end{align*}

where  $\mathcal{I}(p_1,p_2)=\mathcal{E}^{p_1}\cap \mathcal{E}^{p_2}$, $\mathcal{D}(p_1,p_2)=(\mathcal{E}^{p_1}\cup \mathcal{E}^{p_2})-(\mathcal{E}^{p_1}\cap \mathcal{E}^{p_2})$, $\lambda_{p_1}=\prod_{\{i,j\}\in p_1} z_{ij}\omega_{ij}$, $\lambda_{p_2}=\prod_{\{i,j\}\in p_2} z_{ij}\omega_{ij}$, $z_{ij}\sim\mathrm{Bernoulli}(1-p_\text{drop})$, and we have assumed all $\omega_{ij}$ deterministic.
\end{proof}

\section{Experimental results}

Experimental evaluation of GNNs is a topic that has recently received great attention. The emerging consensus is that benchmarking methods routinely used in past literature are inadequate and lack reproducibility. In particular, \cite{vignac2020choice} showed that commonly used citation network datasets like CORA, CITESEER, PUBMED are too simple and skew results towards simpler architectures or even promote ignoring the underlying graph. TU datasets are also recognized to be too small \cite{errica2019fair} and the high variability across splits does not allow for sound comparisons across methods. In order to evaluate the gains offered by randomly wired architectures across, we adopt recently proposed benchmarking frameworks such as the one in \cite{dwivedi2020benchmarking} and Open Graph Benchmarks \cite{hu2020ogb}.

First, we use two datasets in \cite{dwivedi2020benchmarking} to analyse the performance differences between the baseline ResNet architecture, i.e., a feedforward architecture with skip connections after every layer, and the randomly wired architecture. We omit results on architectures without skip connections as these have already been shown to have poorer performance \cite{dwivedi2020benchmarking}.
We focus on the ZINC and CIFAR10 datasets. ZINC is one of the most popular real-world molecular datasets, and considers the task of property regression (constrained solubility) for the molecules represented as graphs. CIFAR10 is a well-known dataset for image classification, and, in this context, images are described by graphs of superpixels. 
For this experiment, we test four of the most commonly used graph convolution definitions: GCN \cite{kipf2016semi}, GIN \cite{xu2018powerful}\footnote{GIN and RAN-GIN compute the output as in \cite{xu2018representation}, using all architecture nodes.}, Gated GCN \cite{bresson2017residual}, and GraphSage \cite{hamilton2017inductive}. Notice that we do not attempt to optimize a specific method, nor we are interested in comparing one graph convolution to another. A fair comparison is ensured by running both methods with the same number of trainable parameters and with the same hyperparameters, keeping exactly the same ones used in \cite{dwivedi2020benchmarking}. The learning rate of both methods is adaptively decayed between $10^{-3}$ and $10^{-5}$ and the stopping criterion is validation loss not improving for 5 epochs after reaching the minimum learning rate. Results are averaged over 4 runs with different weight initialization and different random architecture graphs, drawn with $p=0.6$. The random architectures use sequential paths (Sec. \ref{sec:receptive}), but no DropPath (Sec. \ref{sec:droppath}) for the ZINC experiment, and DropPath but no sequential paths for CIFAR10.

The results in Tables \ref{table:ZINC} and \ref{table:CIFAR} show the performance achieved by randomly wired architectures and their ResNets counterparts for increasing model capacity (number of architecture nodes or layers $L$). We can notice that randomly wired GNNs have compelling performance in many regards. The superscript reports the standard deviation among runs and the subscript reports the level of significance by measuring how many baseline standard deviations the average value of the random architecture deviates from the average value of the baseline. Results are in bold if they are at least $1\sigma$ significant. First of all, randomly wired GNNs typically provide lower error or higher accuracy than their ResNet counterparts for the same number of parameters. Moreover, they are more effective at increasing capacity than stacking layers: while they are essentially equivalent to ResNets for very short networks (e.g., for $L=4$), they enable larger gains when additional layers are introduced. This is highlighted by Table \ref{table:gain}, which shows the relative improvement in mean absolute error or accuracy averaged over all the graph convolution definitions, with respect to the short 4-layer network, where random wiring and ResNets are almost equivalent. This table highlights that deeper ResNets always provide smaller gains with respect to their shallow counterpart than the randomly wired GNNs. This allows us to conclude that randomly wired GNNs are a more effective way of increasing model capacity.

\begin{table}
\caption{ZINC Mean Absolute Error.} \label{table:ZINC}
\setlength\tabcolsep{1.5pt} 
\renewcommand{\arraystretch}{1.5}
\centering
\begin{tabular}{ccccc}
     & $L=4$ & $L=8$ & $L=16$ & $L=32$ \\ \hline
GCN & $\textbf{0.469}^{\pm 0.002}_{2.9\sigma}$ & $0.465^{\pm 0.012}$ & $0.445^{\pm 0.022}$ & $0.426^{\pm 0.011}$ \\
\textbf{RAN-GCN} & $0.509^{\pm 0.015}$ & $\textbf{0.447}^{\pm 0.019}_{1.5\sigma}$ & $\textbf{0.398}^{\pm 0.015}_{2.1\sigma}$ & $\textbf{0.385}^{\pm 0.015}_{3.7\sigma}$\\ \hline
GIN & $0.375^{\pm 0.014}_{0.4\sigma}$ & $0.444^{\pm 0.017}$ & $0.461^{\pm 0.022}$ & $0.633^{\pm 0.089}$\\
\textbf{RAN-GIN} & $0.381^{\pm 0.021}$ & $\textbf{0.398}^{\pm 0.004}_{2.7\sigma}$ & $\textbf{0.426}^{\pm 0.020}_{1.6\sigma}$ & $\textbf{0.540}^{\pm 0.155}_{1.0\sigma}$ \\ \hline
GatedGCN & $0.368^{\pm 0.007}$ & $0.339^{\pm 0.027}$ & $0.284^{\pm 0.014}$ & $0.277^{\pm 0.025}$\\
\textbf{RAN-GatedGCN} & $0.364^{\pm 0.007}_{0.5\sigma}$ & $\textbf{0.310}^{\pm 0.010}_{1.1\sigma}$ & $\textbf{0.218}^{\pm 0.017}_{4.7\sigma}$ & $\textbf{0.215}^{\pm 0.025}_{2.5\sigma}$ \\ \hline
GraphSage & $0.428^{\pm 0.007}_{0.1\sigma}$ & $0.363^{\pm 0.005}$ & $0.355^{\pm 0.003}$ & $0.351^{\pm 0.009}$\\
\textbf{RAN-GraphSage} & $0.429^{\pm 0.010}$ & $\textbf{0.368}^{\pm 0.015}_{1.0\sigma}$ & $\textbf{0.340}^{\pm 0.009}_{5.0\sigma}$ & $\textbf{0.333}^{\pm 0.008}_{2.0\sigma}$ \\ \hline
\end{tabular}
\end{table}

\begin{table}[]
\caption{CIFAR10 Accuracy.} \label{table:CIFAR}
\setlength\tabcolsep{2pt} 
\renewcommand{\arraystretch}{1.5}
\centering
\begin{tabular}{ccccc}
     & $L=4$ & $L=8$ & $L=16$ & $L=32$ \\ \hline
GCN & $54.28^{\pm 0.35}$ & $54.85^{\pm 0.20}$ & $54.74^{\pm 0.52}$ & $54.76^{\pm 0.53}$ \\
\textbf{RAN-GCN} & $\textbf{55.31}^{\pm 0.25}_{2.9\sigma}$ & $\textbf{57.81}^{\pm 0.08}_{14.8\sigma}$ & $\textbf{57.29}^{\pm 0.44}_{4.9\sigma}$ & $\textbf{58.49}^{\pm 0.21}_{7.0\sigma}$ \\ \hline
GIN & $\textbf{70.66}^{\pm 0.78}_{2.94\sigma}$ &  $66.67^{\pm 0.73}$ & $63.99^{\pm 1.45}$ & $58.18^{\pm 2.92}$\\
\textbf{RAN-GIN}& $67.48^{\pm 1.08}$ & $\textbf{67.36}^{\pm 0.70}_{1.0\sigma}$ & $\textbf{67.25}^{\pm 0.74}_{2.2\sigma}$ & $\textbf{62.73}^{\pm 1.57}_{1.6\sigma}$ \\ \hline
GatedGCN & $\textbf{69.26}^{\pm 0.36}_{2.0\sigma}$ & $68.27^{\pm 0.80}$ & $69.16^{\pm 0.66}$ & $69.46^{\pm 0.47}$\\
\textbf{RAN-GatedGCN} & $68.55^{\pm 0.03}$ & $68.86^{\pm 1.64}_{0.7\sigma}$ & $\textbf{72.00}^{\pm 0.44}_{4.3\sigma}$ & $\textbf{73.50}^{\pm 0.68}_{8.6\sigma}$ \\ \hline
GraphSage & $\textbf{66.14}^{\pm 0.21}_{2.4\sigma}$ & $65.58^{\pm 0.46}_{0.6\sigma}$ & $66.12^{\pm 0.11}_{0.0\sigma}$ & $65.33^{\pm 0.34}$\\
\textbf{RAN-GraphSage} & $65.02^{\pm 0.47}$ & $65.31^{\pm 0.38}$ & $66.10^{\pm 1.11}$ & $\textbf{67.68}^{\pm 0.37}_{6.9\sigma}$ \\ \hline
\end{tabular}
\end{table}

\begin{table}
\caption{Median Relative Gain over $L=4$.} \label{table:gain}
\setlength\tabcolsep{3pt} 
\renewcommand{\arraystretch}{1.5}
\centering
\begin{tabular}{ccccc}
                        &        & $L=8$                & $L=16$                & $L=32$                \\ \hline
\multirow{2}{*}{ZINC}   & ResNet & $-0.50\%$            & $+0.37\%$            & $+3.24\%$            \\ \cline{2-5}
                        & \textbf{Random} & $\textbf{+5.43\%}$  & $\textbf{+8.89\%}$   & $\textbf{+11.36\%}$   \\ \hline
\multirow{2}{*}{CIFAR10}& ResNet & $-1.14\%$            & $-0.08\%$            & $-0.47\%$            \\ \cline{2-5}
                        & \textbf{Random} & $\textbf{+0.45\%}$  & $\textbf{+2.62\%}$   & $\textbf{+4.92\%}$   \\ \hline         
\end{tabular}
\end{table}

\begin{table}[t]
\centering
\caption{\textrm{ogbg-molpcba} Average Precision.} 
\label{table:molpcba}
\setlength\tabcolsep{2pt} 
\renewcommand{\arraystretch}{1.5}
\begin{tabular}{cccc}
                         & Test AP               & Val AP                & No. params  \\ \hline
GCN                      & $0.2020^{\pm 0.0024}$ & $0.2059^{\pm 0.0033}$ & $565,928$   \\ \hline
GIN                      & $0.2266^{\pm 0.0028}$ & $0.2305^{\pm 0.0027}$ & $1,923,433$ \\ \hline
ChebNet                  & $0.2306^{\pm 0.0016}$ & $0.2372^{\pm 0.0018}$ & $1,475,003$ \\ \hline
SIGN                     & $0.2047^{\pm 0.0036}$ & $0.2163^{\pm 0.0022}$ & $5,516,228$ \\ \hline
\textbf{RAN-GIN}         & $0.2493^{\pm 0.0076}$ & $0.2514^{\pm 0.0093}$ & $1,868,774$ \\ \hline
GCN+VN+FLAG              & $0.2384^{\pm 0.0037}$ & $0.2556^{\pm 0.0040}$ & $2,017,028$ \\ \hline
GIN+VN+FLAG              & $0.2834^{\pm 0.0038}$ & $0.2912^{\pm 0.0026}$ & $3,374,533$ \\ \hline
DeeperGCN+VN+FLAG        & $0.2842^{\pm 0.0043}$ & $0.2952^{\pm 0.0029}$ & $5,550,208$ \\ \hline
\textbf{RAN-GIN+VN+FLAG} & $0.2879^{\pm 0.0048}$ & $\textbf{0.3041}^{\pm 0.0031}$ & $5,572,026$ \\ \hline
GINE++VN                 & $\textbf{0.2917}^{\pm 0.0015}$ & $\textbf{0.3065}^{\pm 0.0030}$ & $6,147,029$ \\ \hline
\end{tabular}
\end{table}

Moreover, we compare the proposed method against other state-of-the-art techniques to build graph neural networks, including methods that address the oversmoothing problem to build deeper GNNs (DeeperGCN) \cite{li2020deepergcn} or argue that going wide instead of deep is more effective to increase the capacity (SIGN) \cite{rossi2020sign}. This experiment is done on the ogbg-molpcba dataset from Open Graph Benchmarks \cite{hu2020ogb} and results are taken from the public leaderboard. We use a randomly wired version of GIN and compare results with two different setups: a vanilla RAN-GIN with the same number of parameters as GIN, and a larger RAN-GIN using the virtual node trick \cite{gilmer2017neural} and FLAG augmentations \cite{kong2020flag}. Both versions additionally use DropPath with $p_{drop}=0.01$. The results are reported in Table \ref{table:molpcba}. We can see that RAN-GIN (12 layers) significantly outperforms GIN and a number of other techniques for a comparable number of parameters. Furthermore, RAN-GIN with virtual node and FLAG augmentations reaches state-of-the-art performance on this benchmark, outperforming DeeperGCN and being very close to the recently proposed GINE \cite{brossard2020graph}\footnote{Notice that we did not test RAN-GINE.}.

\section{Ablation experiments}

In this section, we explore how some of the design choices for randomly wired GNNs can affect model performance.

\subsection{Architecture Edge probability}

We first investigate the impact of the probability $p$ of drawing an edge in the random architecture. Table \ref{table:p_ablation} shows the results for a basic random architecture without DropPath nor embedded sequential path. It appears that an optimal value of $p$ exists that maximizes performance. This could be explained by a tradeoff between size of receptive field and the ability to modulate it.

\begin{table}[t]
    \centering
    \caption{Edge probability, $L=16$, RAN-GCN.}
    \renewcommand{\arraystretch}{1.5}
    \begin{tabular}{ccccc}
         & $p=0.2$ & $p=0.4$ & $p=0.6$ & $p=0.8$ \\ \hline
        ZINC & $0.440^{\pm 0.025}$ & $0.427^{\pm 0.025}$ & $\textbf{0.409}^{\pm 0.010}$ & $0.415^{\pm 0.012}$ \\
        CIFAR10 & $56.53^{\pm 0.61}$ & $56.21^{\pm 0.48}$ & $\textbf{57.44}^{\pm 0.46}$ & $56.06^{\pm 0.48}$ \\ \hline
    \end{tabular}
    \label{table:p_ablation}
    \vspace{-1pt}
\end{table}

\begin{table}[t]
    \centering
    \caption{DropPath on CIFAR10, RAN-GatedGCN. No sequential path embedding.} \label{table:droppath_ablation}
\setlength\tabcolsep{2pt} 
\renewcommand{\arraystretch}{1.5}
\centering
\begin{tabular}{ccccc}
         & $L=8$ & $L=16$ & $L=32$ \\ \hline
        None & $68.07^{\pm 0.94}$ & $70.78^{\pm 0.38}$ & $72.75^{\pm 0.37}$\\
        DropPath &$\textbf{68.86}^{\pm 1.64}$ & $\textbf{72.00}^{\pm 0.44}$ & $\textbf{73.50}^{\pm 0.68}$ \\ \hline
    \end{tabular}
\end{table}
    
\begin{table}[]
    \centering
    \caption{DropPath on CIFAR10, RAN-GatedGCN. No sequential path embedding.} \label{table:pdrop_ablation}
\setlength\tabcolsep{2pt} 
\renewcommand{\arraystretch}{1.5}
\begin{tabular}{ccccc}
          \multicolumn{5}{c}{$p_\text{drop}$}                                                                             \\
          0                  & 0.005              & 0.01                        & 0.02               & 0.03               \\ \hline
 $70.78^{\pm 0.38}$ & $70.90^{\pm 0.46}$ & $\textbf{72.00}^{\pm 0.44}$ & $71.55^{\pm 0.83}$ & $71.09^{\pm 1.79}$ \\ \hline
\end{tabular}

\end{table}

\begin{table}[t!]
\centering
\caption{Sequential path embedding on ZINC, RAN-GatedGCN. No DropPath.} \label{table:linear_ablation}
\setlength\tabcolsep{2pt} 
\renewcommand{\arraystretch}{1.5}
\vspace{-11pt}
\centering
\small
\begin{tabular}{ccccc}
         & $L=8$ & $L=16$ & $L=32$ \\ \hline
        Fully random & $0.332^{\pm 0.027}$ & $0.264^{\pm 0.029}$ & $0.234^{\pm 0.030}$ \\
        Random+Sequential & $\textbf{0.310}^{\pm 0.010}$ & $\textbf{0.218}^{\pm 0.017}$ & $\textbf{0.215}^{\pm 0.025}$ \\ \hline
    \end{tabular}
\end{table}

\subsection{DropPath}

The impact of DropPath on CIFAR10 is shown in Table \ref{table:droppath_ablation}. We found the improvement due to DropPath to be increasingly significant for a higher number of architecture nodes, as expected due to the increased number of edges. The value of the drop probability $p_\text{drop}=0.01$ was not extensively cross-validated. However, Table \ref{table:pdrop_ablation} shows that higher drop rates typically lowered performance.

\subsection{Embedded sequential path}

The impact of embedding a sequential path as explained in Sec. \ref{sec:receptive} is shown in Table \ref{table:linear_ablation}. It can be observed that its effect of promoting receptive fields with larger radius is useful on this task for any number of architecture nodes. We remark that, while we do not report results for the sake of brevity, this is not always the case and some tasks (e.g., CIFAR10) do not benefit from promoting larger receptive fields.

\section{Conclusion}

We showed how randomly wired architectures can boost the performance of GNNs by merging receptive fields of multiple size. Consistent and statistically significant improvements over a wide range of tasks and graph convolutions highlight how such constructions are more effective at increasing model capacity than building deep GNN by stacking several layers in a linear fashion, even when residual connections are used.

\bibliographystyle{IEEEtran}
\bibliography{biblio}

\end{document}